\documentclass{article}

\usepackage{PRIMEarxiv}

\usepackage[utf8]{inputenc} 
\usepackage[T1]{fontenc}    
\usepackage{hyperref}       
\usepackage{url}            
\usepackage{booktabs}       
\usepackage{amsfonts}       
\usepackage{nicefrac}       
\usepackage{microtype}      
\usepackage{lipsum}
\usepackage{fancyhdr}       
\usepackage{graphicx}       
\graphicspath{{media/}}     

\usepackage{algorithm}
\usepackage{algpseudocode}
\usepackage{cite}
\usepackage[numbers,sort&compress]{natbib}
\usepackage{caption}
\usepackage{lipsum,booktabs}
\usepackage{float}
\usepackage{graphicx}
\usepackage{xcolor}
\usepackage{tabularx,booktabs}
\usepackage[justification=centering]{caption}
\usepackage{url}
\usepackage{amsmath,amssymb,amsthm}
\DeclareMathOperator\erfc{erfc}
\usepackage{graphicx}
\usepackage{float}
\newtheorem{observation}{Observation ---}

\newtheorem{theorem}{Theorem}
\usepackage{graphicx}

\title{Cyclical Curriculum Learning}

\author{
  H. Toprak Kesgin, M. Fatih Amasyali \\
  \textit{Department of Computer Engineering} \\
  \textit{Yildiz Technical University} \\
  Istanbul, Turkey\\
  \texttt{\{tkesgin, amasyali\}@yildiz.edu.tr}
}

\begin{document}
\maketitle

\begin{abstract}
Artificial neural networks (ANN) are inspired by human learning. However, unlike human education, classical ANN does not use a curriculum. Curriculum Learning (CL) refers to the process of ANN training in which examples are used in a meaningful order. When using CL, training begins with a subset of the dataset and new samples are added throughout the training, or training begins with the entire dataset and the number of samples used is reduced. With these changes in training dataset size, better results can be obtained with curriculum, anti-curriculum, or random-curriculum methods than the vanilla method. However, a generally efficient CL method for various architectures and data sets is not found. In this paper, we propose cyclical curriculum learning (CCL), in which the data size used during training changes cyclically rather than simply increasing or decreasing. Instead of using only the vanilla method or only the curriculum method, using both methods cyclically like in CCL provides more successful results. We tested the method on 18 different data sets and 15 architectures in image and text classification tasks and obtained more successful results than no-CL and existing CL methods. We also have shown theoretically that it is less erroneous to apply CL and vanilla cyclically instead of using only CL or only vanilla method. The code of Cyclical Curriculum is available at https://github.com/CyclicalCurriculum/Cyclical-Curriculum.
\end{abstract}


\section{Introduction}
While we learn during our education, we use a specific subject order, namely a curriculum.
We begin with basic, simple information and proceed with more challenging, complex issues.
Generally, experts create this curriculum. Following a curriculum helps humans and animals learn better \cite{humancur,humancur2} . 

Artificial neural networks (ANN), a machine-learning algorithm, were inspired by the communication of biological neuron cells.
However, unlike human sequential learning, the data set in neural network training is randomly ordered in the classical method.
In ANN training, it has been demonstrated that ordering the data in a specific order rather than randomly can improve the success of ANNs \cite{bengio2009curriculum,hacohen2019power,mentornet}. 
This method is known as curriculum learning (CL).
The ordering here can be thought of as simple to difficult or basic to complex.

ANNs have many hyperparameters for training.
Learning rate, epoch number, batch size, and hidden layer number are some of them. While performing the training of ANNs, the classical method divides the whole training set according to a certain number of batches.
It updates the weights of the model with these batches to optimize the determined loss function.
One epoch is completed when weights are updated using all data according to the determined batch size.
When the weights have been updated for the specified number of epochs, the training is complete.
We will call this classical method used as vanilla method in this paper.

In curriculum approaches, the data is sorted according to a specific criterion first.
This order can be from easy to difficult, from hard to easy, or randomly.
After the ranking is determined, the data set is divided according to specific parts.
After the splitting process is performed, optimization is applied for the number of epochs chosen for the first part of the data, and the weights of the model are updated.
This subset is enlarged after a certain number of epochs, and the training continues over these merged sets.
The training is completed when the whole data set is used, and the number of updates is made for the determined epoch number.

The rank can be determined using a variety of methods.
An expert can determine the rank.
Any machine learning or ANN-generated estimated probability values could be used for ranking.
Alternatively, as in the self-taught algorithm, the model itself is trained with the vanilla method, and the loss values produced for the training samples can be used for ranking the samples.


There are studies that have achieved successful results using curriculum \cite{bengio2009curriculum, hacohen2019power, CLNLU, transferCL}, anti-curriculum \cite{NMTCL, EmpiricalNMT, DomainNMT} and random-curriculum \cite{wu2020curricula}. The common point of these studies is to achieve more successful results in training by increasing the size of the training set.

Of course, the training set size does not have to start small to go up. For example,  the training can start with the whole data set, and after a specific number of epochs, the subsets of the training set can be started to use. By doing this, more successful results could be obtained compared to the stochastic gradient descent (SGD) algorithm \cite{zhou2020curriculum}.

Randomly ordered samples can perform as well as or better than curricula and anti-curricula, and better results can be obtained as the training dataset dynamically increases and decreases over time, implying that any benefit is entirely due to the dynamic training set size \cite{wu2020curricula}. 

We compared popular CL methods on a large number of datasets and deep learning architectures in the text and image domains, and we found that current CL methods do not increase success much, in line with the results of studies \cite{wu2020curricula, hacohen2019power} recently published in the literature. Almost all of the methods in the literature increase the size of the dataset \cite{bengio2009curriculum, hacohen2019power, CLNLU}. There are few studies that reduce it \cite{zhou2020curriculum}. Our motivation is to investigate the effects of cyclical training set size rather than simply increasing or decreasing it by combining these two approaches. We propose a cyclical dataset size approach and found that this approach significantly improves success over existing methods.


This paper is divided into the following sections. Section 2 includes  previous studies about curriculum learning. In the section 3, the proposed method, the CCL algorithm, is introduced. The datasets, models and the experimental results are explained in section 4. In Section 5, the theoretical explanation for the CCL algorithm is presented and finally, the results and discussions are included in the section 6.

\section{Literature Review}
It has been shown that there can be significant improvements when artificial neural networks are trained in a meaningful order of samples \cite{bengio2009curriculum} . It has been claimed that this improvement can provide better generalization and increase the speed of non-convex optimization and find a better local minimum for convex optimization. The study was supported with artificial data sets classification and language modeling tasks.

Self-Paced Learning (SPL) \cite{kumar2010self} determines the order (easiness) of training samples during training rather than using a fixed curriculum as in previous studies. In other words, the difficulty values vary with each iteration. The selected subset size is gradually increased by changing the easy samples’ threshold value. Training continues until the entire data set is used. They have experimentally demonstrated the success of their work in four different data sets.

Another competitive model or the model itself (self-thought) can be used to rank samples for curriculum learning\cite{hacohen2019power} .  It has shown similar improvements in both methods. The scoring function for the data set and the pacing function for determining the size of the samples to be given to the network were defined. They demonstrated their success in training speed and test set performance in object recognition tasks.

The dynamic instance hardness is another method for determining dynamic difficulty (DIH) \cite{zhou2020curriculum}. DIH is concerned not only with the final loss values of the samples, but also with how this loss value changes during training. It is calculated during training, just like in SPL \cite{kumar2010self}, but it takes into account not only the final value but also the difficulty values from previous iterations. Training begins with all training samples. After a certain number of epochs, the training set size is reduced, and this set is chosen from samples with high DIH values. 


Cyclical Learning Rate is proposed for the learning rate, one of the critical hyperparameters of artificial neural networks. \cite{smith2017cyclical} In this method, instead of gradually decreasing the learning rate, it has been shown that changing the learning rate cyclically between reasonable values gives more successful results in the test set for accuracy metric. The study has suggested practical methods for determining the reasonable learning rate range. They demonstrated their work using various networks in cifar10 and cifar100.

In \cite{wu2020curricula}, Extensive experiments were conducted on curriculum learning and examined the situations in which curriculum methods work.
Curriculum, Anti-Curriculum, and Random-Curriculum methods were compared. The Random Curriculum method has performed as well as or better than other methods. In this case, it has been shown that the improvement is due to the change in the size of the data set used in training. In the study, also the success of curriculum learning in noisy data and limited-time training was examined. Their experiments have shown that curriculum methods are more successful with noisy data and limited time for training.

\section{Proposed Method}

In this section, the cyclical data set sizes, how the samples are selected for training, how the scores are determined, and the training algorithm are explained.

\subsection{Cyclical Training Dataset Sizes}

The vanilla method uses the entire training set in one epoch. CCL, on the other hand, only uses a subset of the training dataset in one epoch. The size of these subsets is determined by parameters. The samples to be chosen for the subset are determined by certain scores.

Hyperparameters of the cyclical training datasets sizes algorithm are initial percent, final percent, alpha, and epoch count. The starting percentage indicates what percentage of the data the cycle will start with, and the final percentage indicates the point it will increase. The alpha parameter determines how fast or slow the data set size changes will be between 0 and 1. The algorithm for the cycling approach that determines the percentage of the subsets is given in Algorithm 1.

\begin{algorithm}
\caption{Get Dataset Sizes}
\begin{algorithmic}[1]
\Require

\State $T$ : epoch count for training.
\State $sp$ : initial percentage of data set size.
\State $ep$ : final percentage of data set size.
\State $\alpha$ : speed of cycle.
\Statex
\Procedure {get\_dataset\_sizes}{$sp$, $ep$, $\alpha$, $T$}
\State $S \leftarrow $  [ ] (initialize empty list)
\State $n \leftarrow sp$
\State $S.append(n)$
\For {$t \in \{1...(T-1)\}$}
\If {(n == sp) or (($S_{t-1} < S_t$) and (n!= ep) )}
\State $n \leftarrow min((n * (1 / \alpha)),ep) $
\Else
\State $n \leftarrow max((n * (\alpha)),sp) $

\EndIf
\State $S.append(n)$
\EndFor

\State \Return $S$

\EndProcedure

\end{algorithmic}
\end{algorithm}

\begin{figure}[h]
\centering
\includegraphics[width=\columnwidth]{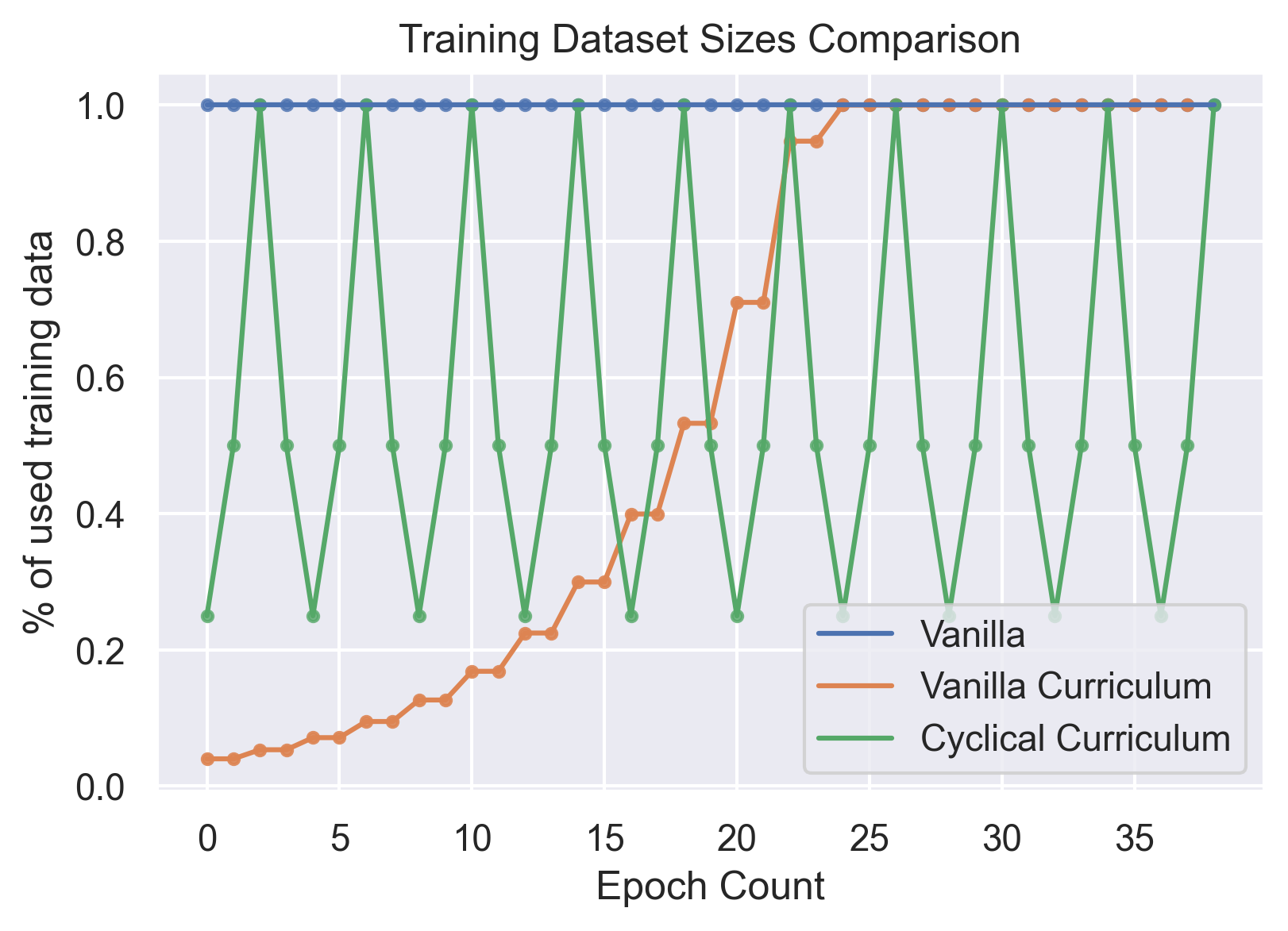}
\caption{Training Dataset Sizes Comparison}
\end{figure}

Figure-1 shows how the size of the training dataset changes throughout the training. For Cyclical Dataset Sizes, initial percent (sp) = 0.25, final percent (ep) = 1.0, alpha = 0.5 was used.

\subsection{Selecting Samples}

Since a different subset of the training set will be used for each epoch of CCL, the number of samples should be selected as much as the specified subset size for model training.
For this selection process, various algorithms can be used;

1) Samples for the subset can be randomly selected at each step.

2) All samples are sorted according to a specific criterion, and the top n are selected.

3) Probability values are calculated for all samples. And these probability values are used to select samples.




In this study, when determining the subsets of the data set, we used a probabilistic selection algorithm. We examined approaches to picking easy samples with higher probability, selecting difficult ones with higher probability, and randomizing them. We observed that using easy samples for the subset is more successful than other methods. And also we found that the probabilistic selection works better than greedy selecting the top n samples. Therefore, in this study's experiments, we used the third method, the probabilistic selection algorithm, with the sampling below. 

$ p_i \sim r(i) $ where $r(i)$ is probability scores. Training data indices are sampled with a probability proportional to the inverse loss values. Algorithm-1 determines the percentage of indices to be selected per epoch. Sampling is done without replacement.


\subsection{Determining Scores}
In order to make a probabilistic selection, all data must have a score value.
We obtained these scores by using the model itself (self-thought). First, the model is trained for a certain number of epochs. Then the model makes predictions for the training set. From these estimates, the loss value is calculated for each sample. In order to obtain probability values from these loss values, the following operations are applied. Since the small loss value represents the easy example, the calculated loss values are reversed according to the multiplication. All scores are then divided by the total number of scores for normalization.
Formally, the score(probability) values for each sample can be calculated as $k_i = \frac{1}{l(y_i,M_1(X_i,W)))} $ and                        
$r_i = \frac{k_i}{\sum k_i} $ 
where $l$ is loss function, $ M_1 $ is trained model, $W$ is final model parameters for $ M_1 $ and $r$ is score values for each sample. Then, the samples are selected with the probabilistic selection algorithm. (See section 3.2)

\subsection{Training}
After the training set sizes and scores are determined, the training is performed according to Algorithm-2.






\begin{algorithm}
\caption{Training}
\begin{algorithmic}[1]
\Require
\State $D$ : Training dataset consists of X: Input and y: Output.
\State $M_1$ : the model by which the scores will be obtained.
\State $M_2$ : the model to train.
\State $T$ : epoch count for training.
\State $sp$ : initial percentage of data set size.
\State $ep$ : end percentage of data set size.
\State $\alpha$ : speed of cycle.
\Statex

\Procedure {train}{}
\State \Call{train\_model}{$M_1$,$D$} (Train the scoring model with entire training data.)
\State $S \leftarrow \Call{get\_sizes}{ep,sp,\alpha,T} $ (Section 3.1)
\State $r \leftarrow \Call{get\_scores}{M_1,D} $ (Section 3.3)
\State $M_2 \leftarrow  \Call{initilize\_model}{} $ (Initilize the model with random weights)
\For {$t \in \{1...T\}$}
\State $n \leftarrow \Call{select\_by\_scores}{D,S[t],r} $ (Section 3.2)
\State \Call{update\_model}{$M_2$,$D(n)$}

\EndFor

\EndProcedure
\end{algorithmic}
\end{algorithm}



The $get\_sizes$ procedure is as described in section 3.1. It returns a list (S) such as $[0.25, 0.5, 1.0, 0.5, 0.25, 0.5, 1.0]$ . When $S[i]$ = 1, since selection is done without replacement, the entire data set is used in step i. This means applying the vanilla method in step i. When $S[i]<1$, a smaller part of the dataset is used and since the selection is made according to the losses of samples, this is to apply the CL algorithm in the step i. Therefore, CCL applies vanilla and CL cyclically.

The CCL algorithm can be easily converted to other CL algorithms. For example, if the S array produced by the $get\_sizes$ function consists of only 1s ($S[i] = 1$) and the $get\_scores$ function produces equal scores for all samples, it becomes a vanilla method. If the array produced by the $get\_sizes$ function is an incremental array and equal scores are used for all instances, the method becomes Random CL. In CCL, on the other hand, the S array produced by the $get\_sizes$ function consists of cyclically changing values and the scores of all samples are determined by a certain model with the $get\_scores$ function.

\section{Experiments}

We ran our experiments with various data sets in order to perform various tasks over different networks.
Experiments can be divided into two main groups, as image classification and text classification.
The datasets used for image classification are cifar-10\cite{krizhevsky2009learning}, cifar-100\cite{krizhevsky2009learning}, fashion mnist\cite{xiao2017fashion}, and stl-10\cite{coates2011analysis}.
Data sets differ from each other in several ways.
Cifar-100 has 100 classes, whereas other datasets have 10 classes.
Cifar10 and cifar100 contain 32x32, stl-10 96x96 color images, while for fashion mnist each example is a 28x28 grayscale image. The cifar10, cifar100, fashion and stl10 datasets contain 50000, 50000, 60000 and 5000 training samples, respectively. And the test sample counts for cifar10, cifar100, fashion and stl10 are 10000, 10000, 10000 and 8000 respectively.


These image datasets are modeled using a basic CNN network. 
The CNN network used can be described as follows.
The convolutional layer is followed by max-pooling twice; then, after flattening, another convolutional layer is applied and transferred to the dense layer.
Following the softmax activation, predictions are generated.
Relu activation function is used in the intermediate layers.
Batch normalization was applied before activation functions.
The image sizes and number of classes in the dataset are used to adjust the input and output layers.

For natural language processing experiments, data set and model pairs were obtained from various sources on the internet \cite{20newsmodel,sarcasmmodel,reutersmodel,hotelmodel,stweetmodel,ctweetmodel,qpairmodel,foodmodel,sofmodel,toxicmodel,redditmodel,squadmodel,imdbmodel,nermodel} .
The datasets used are 20news, sarcasm, reuters, hotel, stweet, ctweet, qpair, food, sof, toxic, reddit, squad, imdb, ner \cite{20news,sarcasm,book,reuters,hotel,stweet,ctweet,qpair,food,sof,toxic,reddit,squad,imdb,ner}. 
Task and model diversity were prioritized when selecting data sets.
News classification, sentiment analysis, sarcasm detection, named-entity recognition are examples of selected tasks.
Various deep learning models were used for the datasets.
These networks were used to test the vanilla, other curriculum methods, and the CCL.

\subsection{Training Process}

We first trained the model until no improvement until two consecutive epochs in the validation set (early stopping) for the required number of epochs.

We generated the scores in this epoch number using the model itself. While conducting the experiments, we trained the models with three times the number of epochs from which we obtained the scores. In the image datasets, we used 128 as the batch size, except for stl-10. We used 32 batch size for stl-10. We used Adam (Adaptive Moment Estimation) \cite{adam} as the optimizer in all image experiments. For text classification tasks, we used the hyperparameters in the sources from which the dataset model pairs were taken.

We determined our evaluation criteria as top-1 accuracy and tested the models in linear intervals during the training. We determined the success criteria as the maximum accuracy achieved during these tests. We started the models from the same starting point for a fair evaluation and repeated the experiments five times with different starting points.





\subsection{Summary of Results}

\begin{table*}[]
\begin{tabular}{llllllllll}
\toprule
         & Vanilla & CCL                          & Anti-CCL                     & Rand-CCL                     & CL                           & Anti-CL                      & Rand-CL                      & SPL                          & Anti-SPL                     
         \\ \midrule
cifar10  & 72.09   & {\color[HTML]{009901} 75.51} & {\color[HTML]{CB0000} 69.83} & {\color[HTML]{CB0000} 70.01} & 72.98                        & {\color[HTML]{CB0000} 66.46} & {\color[HTML]{CB0000} 69.26} & 72.34                        & {\color[HTML]{FF0000} 72.09} \\
cifar100 & 38.77   & {\color[HTML]{009901} 42.45} & {\color[HTML]{CB0000} 35.40} & {\color[HTML]{CB0000} 36.54} & {\color[HTML]{CB0000} 37.05} & {\color[HTML]{CB0000} 31.70} & {\color[HTML]{CB0000} 33.16} & 39.11                        & 38.13                        \\
fmnist   & 90.74   & {\color[HTML]{009901} 91.32} & 90.66                        & 90.64                        & {\color[HTML]{009901} 91.10} & {\color[HTML]{CB0000} 89.83} & {\color[HTML]{CB0000} 90.23} & 90.95                        & 90.96                        \\
stl\_10  & 59.53   & {\color[HTML]{CB0000} 56.94} & {\color[HTML]{CB0000} 56.78} & {\color[HTML]{CB0000} 55.58} & 59.06                        & {\color[HTML]{CB0000} 58.49} & {\color[HTML]{CB0000} 57.90} & 59.61                        & 59.42                        \\
20\_news & 71.65   & {\color[HTML]{009901} 72.43} & {\color[HTML]{CB0000} 69.34} & {\color[HTML]{CB0000} 69.61} & 72.36                        & {\color[HTML]{CB0000} 68.80} & 71.05                        & 71.88                        & 71.23                        \\
sarcasm  & 82.23   & {\color[HTML]{009901} 82.65} & {\color[HTML]{CB0000} 81.33} & {\color[HTML]{CB0000} 81.36} & 81.74                        & 82.10                        & 82.18                        & {\color[HTML]{FF0000} 81.42} & {\color[HTML]{FF0000} 81.77} \\
reuters  & 79.82   & {\color[HTML]{009901} 80.30} & 79.60                        & 79.68                        & {\color[HTML]{009901} 80.63} & {\color[HTML]{CB0000} 79.26} & {\color[HTML]{CB0000} 79.46} & 79.73                        & 79.74                        \\
hotel    & 60.51   & 60.42                        & 59.97                        & {\color[HTML]{CB0000} 59.77} & 60.67                        & {\color[HTML]{CB0000} 59.35} & 60.22                        & 59.07                        & 58.91                        \\
stweet   & 78.02   & 78.03                        & {\color[HTML]{CB0000} 77.82} & 78.03                        & 78.03                        & 77.96                        & 77.96                        & {\color[HTML]{FF0000} 77.79} & {\color[HTML]{FF0000} 76.39} \\
ctweet   & 85.88   & 85.73                        & 85.71                        & 85.49                        & 85.46                        & {\color[HTML]{CB0000} 83.18} & {\color[HTML]{CB0000} 84.49} & 85.66                        & 85.69                        \\
qpair    & 78.44   & 78.28                        & 78.52                        & 78.42                        & 78.07                        & 78.39                        & 78.21                        & {\color[HTML]{FF0000} 77.51} & 77.97                        \\
food     & 94.04   & {\color[HTML]{009901} 94.20} & {\color[HTML]{CB0000} 93.92} & {\color[HTML]{CB0000} 93.99} & {\color[HTML]{009901} 94.18} & {\color[HTML]{CB0000} 93.81} & {\color[HTML]{CB0000} 93.96} & 94.03                        & {\color[HTML]{FF0000} 93.95} \\
sof      & 88.13   & {\color[HTML]{009901} 88.31} & 87.93                        & 88.32                        & 88.16                        & {\color[HTML]{CB0000} 87.10} & {\color[HTML]{CB0000} 87.22} & 88.16                        & {\color[HTML]{FF0000} 87.41} \\
toxic    & 92.75   & {\color[HTML]{009901} 93.31} & 92.71                        & 92.76                        & {\color[HTML]{009901} 93.35} & 92.74                        & 93.29                        & 92.44                        & 92.96                        \\
reddit   & 72.82   & 72.81                        & 72.81                        & 72.78                        & {\color[HTML]{CB0000} 72.64} & 72.80                        & 72.77                        & {\color[HTML]{FF0000} 72.56} & {\color[HTML]{FF0000} 72.41} \\
squad    & 78.11   & 78.26                        & 78.18                        & 77.86                        & 78.66                        & 77.70                        & 78.66                        & 78.72                        & 78.65                        \\
imdb     & 92.11   & 92.24                        & 91.27                        & 92.14                        & {\color[HTML]{CB0000} 91.16} & 91.65                        & 92.10                        & 92.34                        & 92.24                        \\
ner      & 73.40   & {\color[HTML]{009901} 74.25} & 72.96                        & 72.93                        & {\color[HTML]{CB0000} 72.77} & {\color[HTML]{CB0000} 72.82} & 73.38                        & {\color[HTML]{009901} 73.84} & {\color[HTML]{FF0000} 72.37}
\\ \bottomrule
\end{tabular}
\centering
\caption{Average Maximum Accuracy Table \\
In the experiments, the models were started from the same weights. And it has been repeated five times. The highest accuracy scores achieved were averaged. With comparison to vanilla, statistically better ones are shown in green, worse ones are shown in red.}
\centering
\end{table*}

Table-1 displays the average top-1 accuracy across test sets.
The accuracy of the NER (Named Entity Recognition) has calculated without the 'O' tag because the vast majority of tags (\%85) are the 'O' tag among 17 tags.
The Squad dataset is for a question answering task rather than a classification task. 
Therefore, predicted and ground truth answer's exact match percentage are used as success criteria.

In these experiments, we used the same initial weights for all methods. And we repeated experiments five times with different starting points. The average of these 5 values are shown in Table-1. For comparison of methods, we used student's t-test with p = 0.05. We compare all methods with the vanilla method. Statistically better ones are shown in green, worse ones are shown in red. CCL outperforms vanilla in ten datasets out of eighteen. And in only one dataset, its performance is lower. We used the same cycle rate ($sp=0.25$, $ep=1$, $\alpha=0.5$) for all data sets. However, it can be possible to achieve more successful results with the optimization of cycling hyperparameters.

Since a subset of the data set is used in each epoch, different methods are used to select samples for this subset. The cyclical curriculum was selected based on the probability values of the score values. These score values are composed of the loss values from a model trained with a different starting point on the training set. Obtained loss values are reversed (the lower loss, the higher probability) and normalized as a probability value. Cyclical Anti-curriculum means the opposite. Cyclical Random-curriculum refers to random selection according to subset size.  The Curriculum is not the probabilistic selection of subsets, but data set growth using fixed rank according to scores. For example, let the number of samples in the whole data set is 3n. The easiest n samples are used for a certain number of epochs. After that, the easiest 2n samples are used for a given epoch and finally, the entire data set is used. With these subsets it is also important that keeping the sample balanced with the same number of examples from each class as in the training set. The comparison of these methods during the training of cifar 10 datasets is given in Fig.~\ref{fig:comp}. The method that selected more easy samples proportionally to their probabilistic values, CCL, yielded the best results from these methods.

\subsection{Investigating Training Process}

\begin{figure}[H]
\includegraphics[width=\columnwidth]{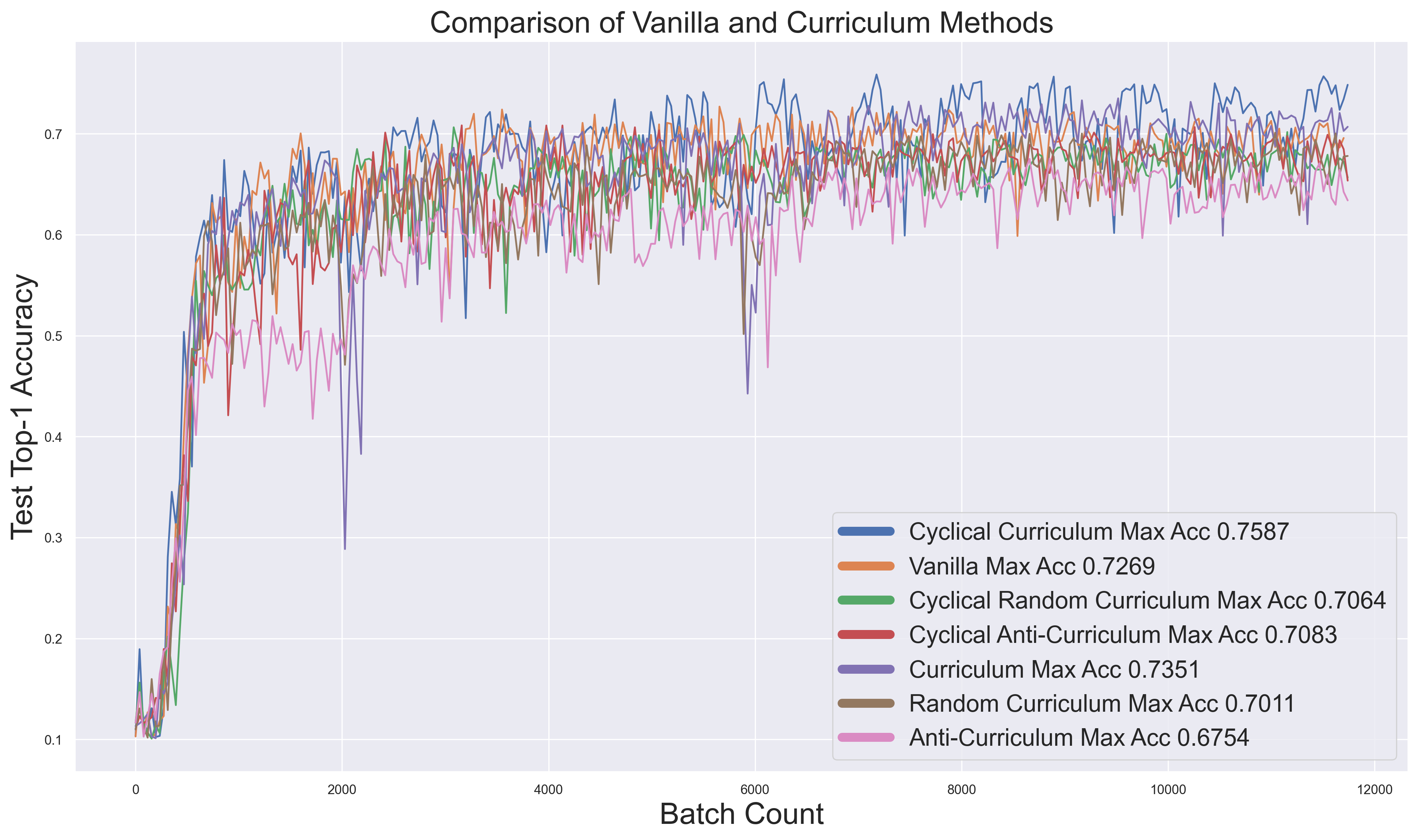}
\caption{Comparison of vanilla, cyclical curriculum and vanilla curriculum methods for $\alpha = 0.5$, initial percent = 0.25, final percent = 1 parameters in cifar 10 data set.}
\label{fig:comp}
\end{figure}

In the experiments, we obtained better results by using the CCL compared to the vanilla method. It can be observed in Figure-2 that CCL achieves a much better minimum point than the vanilla method and other curricula methods. In cyclic curriculum variants trials, we observed that using fixed scores at the end of each epoch gives better results than using dynamic scores as in the SPL \cite{kumar2010self} method. We have seen that the change of training data set size with time gives better results than it is constant ($\alpha = 1$ and $sp! = 1$). However, the rate of change of data set size remains a hyperparameter that needs to be optimized.

\begin{figure}[h]
\includegraphics[width=\columnwidth]{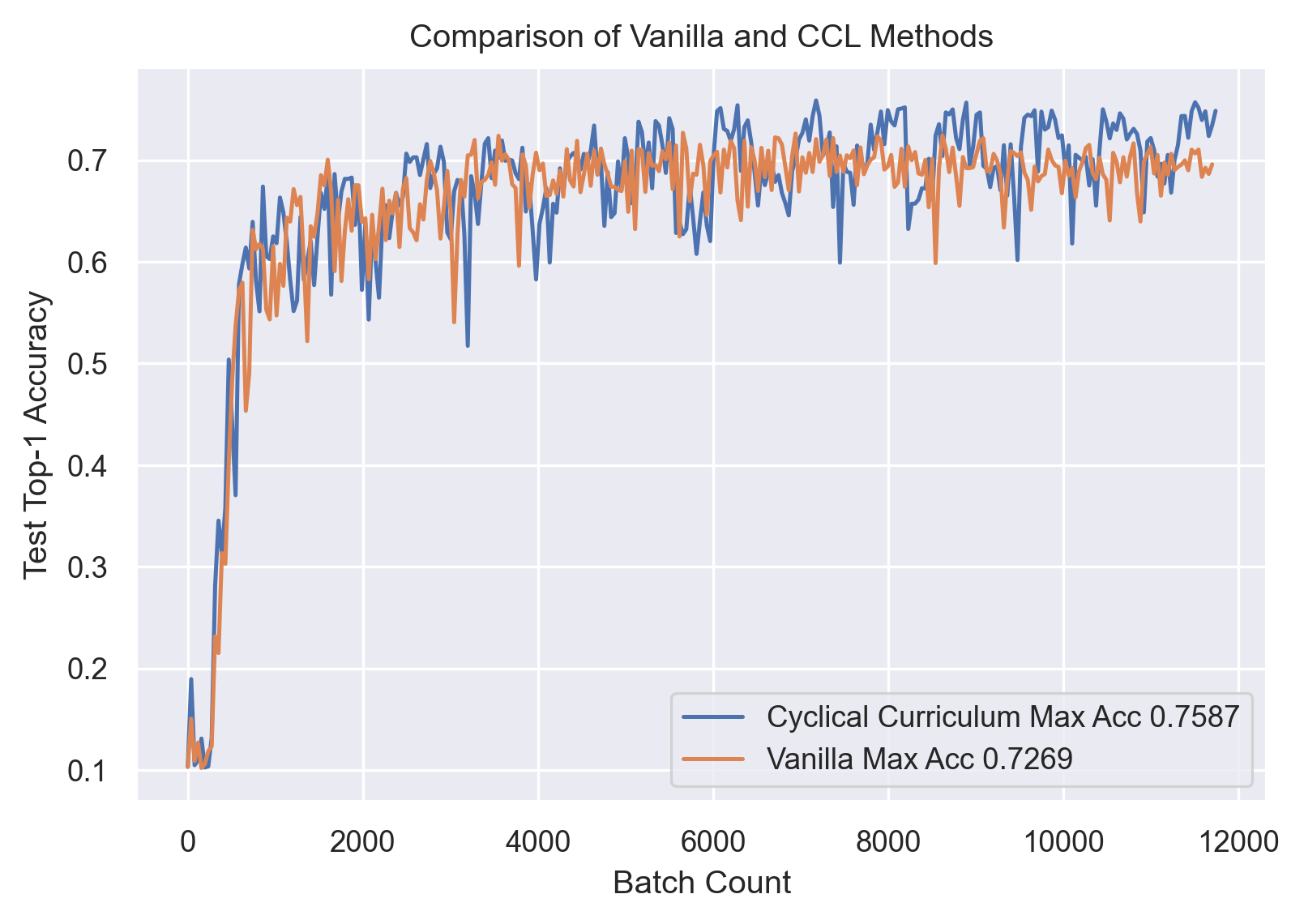}
\caption{Comparison of vanilla and cyclical curriculum methods for $\alpha = 0.5$, initial percent = 0.25, final percent = 1 parameters in cifar 10 data set.}
\end{figure}

The CCL achieves a higher success rate than the vanilla method.  we noticed that there is also a cyclic in the success change in the test dataset. For CCL, the success rate is making lots of ups and downs and oscillating in a wider range. This allows it to achieve more successful results. Figure 3 is an example of this process. This "accuracy cycling" is similar to the \cite{smith2017cyclical} cycle in a study that suggests changing the learning rate to cycles. The theoretical explanation as to why the CCL method is successful will be given in the next section.

\section{Why does Cyclical Curriculum Work?}
Machine learning algorithms, specifically deep learning models, try to minimize a function throughout training. Let the F function be minimized for the D data set ($D = \{x_i, y_i\}_{i=1}^{N} $). Let D dataset consist of N samples and $x_i$ represent the $i^{th}$ training example and $y_i$ its label. The function $F$ to be minimized can be written as follows. 
\begin{flalign}
argmin(F_w) \nonumber
\end{flalign}
\begin{flalign}
F(w) = \frac{1}{N} \sum_{i = 1}^{N} f_i(w) 
\end{flalign}
\begin{flalign}
f_i = l ( y_i, M(x_i,w) ) \nonumber
\end{flalign}
$M(x_i,w)$ is the prediction for $x_i$ of a model with w parameters, $y_i$ is the actual label of the sample i and $l$ is the loss function.

This optimization process can be performed with the gradient descent (GD) algorithm $w_{t+1} = w_t - \gamma \nabla_w F(w_t)$ where w is all parameters of the model (weights of deep learning model) $\nabla_w$ is the derivative of loss function with w parameters and $\gamma$ is learning rate. However, the GD algorithm is not suitable for large-scale problems. Instead, in practice, the optimization is performed using the stochastic gradient descent (SGD) algorithm with the equation $w_{t+1} = w_t - \gamma \nabla_w f_i(w_t)$. The difference of the SGD algorithm from the GD algorithm is that it updates the weights at each step, not according to all samples in the training set, but according to a small number of randomly selected samples at each step. Since SGD does not use all training examples at each step, it makes some errors compared to GD. The error of updating with sample i at time t can be shown with $ error(w_t) = \mid \nabla_w f_i(w_t) - \nabla_w F(w_t) \mid$ .

If the $i^{th}$ sample is selected from a uniform distribution, it represents SGD and calculated by Equation \ref{gd} where $N$ is the number of samples and $r_i$ is $i^{th}$ the sample's probability of being selected. 

\begin{equation} \label{gd}
r_i = \frac{1}{N}
\end{equation}

If the samples are selected according to the Equation 3, we name it as exponential distributed stochastic gradient descent (ESG).

\begin{equation} \label{esg}
r_i = \frac{\lambda . exp(- \lambda . f_i)}{ \sum_{n=1}^{N} \lambda . exp(- \lambda . f_i)}
\end{equation}

Selecting $r_i$ using ESG makes an error with respect to GD, where all examples are used instead of the $i^{th}$ example.
The mean square error made can be described by Equation \ref{err}.

\begin{equation} \label{err}
E[error(w_t)] = E[ ( \nabla_w f_i(w_t) -  \nabla_w F(w_t) )^2 ]
\end{equation}

We can write expected value of the error for SGD with Equation \ref{sgd_er1} and ESG Equation \ref{esg_er1} using variance bias decomposition. For $A = \nabla_w f_i(w_t)$ and $B = \nabla_w F(w_t)$
\begin{equation} \label{sgd_er1}
E_{sgd}[error(w_t)] = (E_{sgd}[A] - B)^2 + E_{sgd}[A^2] - E_{sgd}^2 [A]
\end{equation}

\begin{equation} \label{esg_er1}
E_{esg}[error(w_t)] = (E_{esg}[A] - B)^2 + E_{esg}[A^2] - E_{esg}^2 [A]
\end{equation}

\begin{theorem}
If the loss values come from a normal distribution ${N}(\mu,\,\sigma^{2})$ at time t, the expected error of SGD will be lower than the ESG error.

$f_i(w_t) \sim {N}(\mu,\,\sigma^{2}) $
$\implies$
$E_{sgd}[error(w_t)] < E_{esg}[error(w_t)]$ .

\end{theorem}

\begin{proof}
For $A = \nabla_w f_i(w_t)$  and
$B = \nabla_w F(w_t)$, the expected value of SGD and ESG errors are given in Equation \ref{sgd_er1} and \ref{esg_er1} and can be calculated as follows.

Since $f_i(w_t) \sim N(\mu, \sigma^{2})$ 

\begin{equation} \label{esg_er}
E_{sgd}[A] = \mu
\end{equation}

From definition of $F(w_t)$

\begin{equation} \label{esg_er}
B = \mu
\end{equation}

From the definition of variance

\begin{equation}
E_{sgd}[A^2] - E_{sgd}^{2}[A] = \sigma^2
\end{equation}

Therefore

\begin{flalign} \label{ex_esg_er1}
E_{sgd}[error(w_t)] &= 0 + \sigma^{2} \\ \nonumber
&= \sigma^{2}
\end{flalign}

$E_{esg}[A]$ and $E_{esg}[A^2]$ can be calculated respectively.

\begin{flalign} \label{ex_esg_er}
E_{esg}[A] &= 
\frac{E[\lambda.f_i(w_t).exp(-.\lambda.f_i(w_t))]}{E[\lambda.exp(-\lambda.f_i(w_t))]} \\ \nonumber
&= \frac{\lambda.exp(\frac{\lambda(\lambda.\sigma^{2}-2.\mu)}{2}).(\mu - \lambda.\sigma^{2}) }{\lambda.exp(\frac{\lambda(\lambda.\sigma^{2}-2.\mu)}{2})} \\ \nonumber
&= \mu - \lambda \sigma^{2}
\end{flalign}

\begin{flalign} \label{ex_esg_er}
E_{esg}[A^2] &= 
\frac{E[\lambda.f_{i}^{2}(w_t).exp(-\lambda.f_{i}(w_t))}{E[\lambda.(w_t).exp(-\lambda.f_{i}(w_t))} \\ \nonumber
&= \lambda^{2}.\sigma^{4} - 2.\lambda.\mu.\sigma^{2} + \mu^{2} + \sigma^{2} 
\end{flalign}

\begin{flalign}
E_{esg}[A^2] - E_{esg}^2[A] &= \lambda^{2}.\sigma^{4}-2.\lambda.\mu.\sigma^{2} + \mu^{2} + \sigma^{2}  \\ \nonumber
& - (\mu - \lambda.\sigma^{2})^2  \\ \nonumber
&= \sigma^{2} 
\end{flalign}

Therefore,
\begin{flalign}
E_{esg}[error(w_t)] &= 
(\mu - \lambda.\sigma^{2} - \mu )^2 + \sigma^{2} \\ \nonumber
&= \lambda^{2} \sigma^{4} + \sigma^{2}
\end{flalign}

We calculated the expected value of the SGD error as $E_{sgd}[Error(w_t)] = \sigma^{2}$ and the expected value of the ESG error as $E_{esg}[error(w_t)] = \lambda^{2} \sigma^{4} + \sigma^{2}$. Therefore, we found that the expected value of the SGD error is lower than the expected value of the ESG error.

\begin{flalign}
\sigma^{2}  <  \lambda^{2} \sigma^{4} + \sigma^{2} \implies \nonumber E_{sgd}[Error(w_t)] < E_{esg}[error(w_t)] \\ \nonumber
\end{flalign} 

\end{proof}

\begin{theorem}
If the loss values come from a half-normal distribution at time t and $\sigma\lambda < \pi$ the expected value of ESG error will be lower than the expected value of SGD error.

$f_i(w_t) \sim {HalfNorm}(\mu,\,\sigma^{2})$ $\land$ $\sigma\lambda < \pi$
$\implies$
$E_{esg}[error(w_t)] < E_{sgd}[error(w_t)]$ .

\end{theorem}

\begin{proof}

From Equation 5, we can calculate the expected error of SGD as follows;

Since $f_i(w_t) \sim {HalfNorm}(\mu,\,\sigma^{2})$
\begin{flalign}
E_{sgd}[A] = \mu + \sigma \sqrt{\frac{2}{\pi}} \\ \nonumber
B = \mu + \sigma \sqrt{\frac{2}{\pi}}
\end{flalign}

\begin{flalign}
E_{sgd}[A^{2}] - E_{sgd}^2[A] = \sigma^2 (1 - {\frac{2}{\pi}} )
\end{flalign}

Therefore 

\begin{flalign}
E_{sgd}[Error(w_t)] &= 0^2 + \sigma^2 (1 - \frac{2}{\pi}) \\ \nonumber
&= \sigma^2 (1 - \frac{2}{\pi})
\end{flalign}

From equation 6, we can calculate the expected error of ESG as follows;



\begin{flalign}
E_{esg}[A] &= \frac{E[\lambda\nabla_{w}f_i(w_t).\exp(-\lambda\nabla_{w}f_i(w_t)]}{E[\lambda.\exp(-\lambda\nabla_{w}f_i(w_t)]} \\ \nonumber
&= \frac{\lambda(\mu - \sigma^{2} \lambda) (\exp{(\frac{\sigma^2\lambda^2}{2} - \lambda \mu )}) . \erfc{}(\frac{\sigma. \lambda . \sqrt{2}}{2})}{\lambda (\exp{(\frac{\sigma^2\lambda^2}{2} - \lambda \mu )}) \erfc{}(\frac{\sigma. \lambda . \sqrt{2}}{2})} \\ \nonumber
&+ \frac{\frac{\sigma \sqrt{2} \exp(-\mu.\lambda)}{\sqrt{\pi}}}{\lambda (\exp{(\frac{\sigma^2\lambda^2}{2} - \lambda \mu )}) \erfc{}(\frac{\sigma. \lambda . \sqrt{2}}{2})} \\ \nonumber
&= \mu - \sigma^2\lambda + \frac{\sigma\sqrt{2}.\exp{(\frac{-\sigma^2\lambda^2}{2})}}{\erfc{(\frac{\sigma\lambda\sqrt{2}}{2})} \sqrt{\pi}}
\end{flalign}

We can call the last part of this expression as C for ease of representation.

\begin{flalign}
C = \frac{\sigma\sqrt{2}.\exp{(\frac{-\sigma^2\lambda^2}{2})}}{\erfc{(\frac{\sigma\lambda\sqrt{2}}{2})} \sqrt{\pi}}
\end{flalign}

\begin{flalign}
E_{esg}[A^2] &= \frac{E[\lambda\nabla_{w}f_i^2(w_t).\exp(-\lambda\nabla_{w}f_i(w_t)]}{E[\lambda.\exp(-\lambda\nabla_{w}f_i(w_t)]} \\ \nonumber
&= \sigma^4\lambda^2-2\sigma^2\lambda\mu+\mu^2+\sigma^2-C(\sigma^2\lambda-2\mu) \\ \nonumber
\end{flalign}

\begin{flalign}
E_{esg}[A^2] - E_{esg}^2[A] &= \sigma^4\lambda^2-2\sigma^2\lambda\mu+\mu^2+\sigma^2 \\ \nonumber
&-C(\sigma^2\lambda-2\mu)  - (\mu - \sigma^2\lambda+C)^2 \\ \nonumber
&= \sigma^2+C\sigma^2\lambda-C^2
\end{flalign}

\begin{flalign}
E_{esg}[A] - B &= \mu + \sigma\sqrt{\frac{2}{\pi}} - (\mu - \sigma^2 \lambda + C)  \\ \nonumber
&= \sigma\sqrt{\frac{2}{\pi}} + \sigma^2\lambda - C
\end{flalign}

Therefore

\begin{flalign}
E_{esg}[Error(w_t)] &= (\sigma\sqrt{\frac{2}{\pi}}+\sigma^2\lambda-C)^2 + \sigma^2 + C\sigma^2\lambda-C^2 \\ \nonumber
&= (\sigma^2\lambda-C) (\sigma^2\lambda+2\sigma\frac{2}{\pi}) + \sigma^2(\frac{2}{\pi}+1)
\end{flalign}

\begin{figure}[H]
\centering
\includegraphics[scale=0.6]{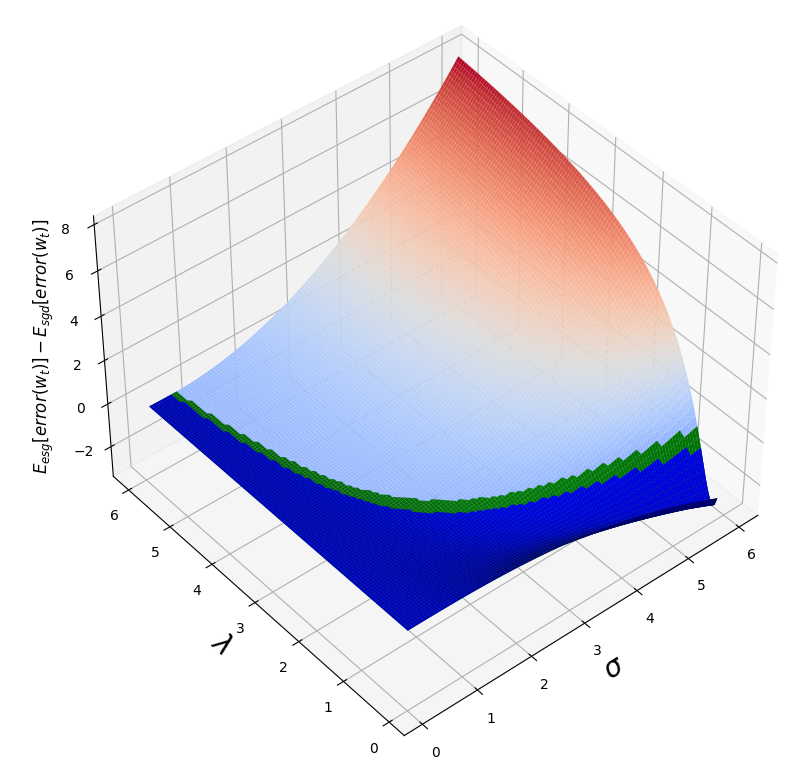}
\caption{$E_{esg}[Error(w_t)] - E_{sgd}[Error(w_t)]$ for $\sigma$ and $\lambda$}
\end{figure}
Since the C term contains the erfc function, its analytic equivalent can not be written. For this reason, we evaluate these expressions numerically. We plotted $E_{esg}[error(w_t)] - E_{sgd}[error(w_t)]$ according to $\sigma$ and $\lambda$ in Figure-4. Figure-4 shows the region where ESG is lower than SGD. The dark-blue and green area shows that the expected value of ESG error is lower than the expected value of SGD error. The lower base of the green area shows the line where $\lambda\sigma=\pi$.




As seen in Figure-4, $\sigma\lambda<=\pi \implies E_{esg}[Error(w_t)] < E_{sgd}[Error(w_t)]$.

\end{proof}

\begin{observation}

In deep learning models, the losses of the samples come from the normal distribution before the training starts. As the training continues, the losses of the samples change to the half-normal distribution.

\end{observation}

\begin{observation}
During training, if the distributions of training sample's losses are half-normal distribution and the model is optimized with ESG, the training sample's losses changes from half-normal distribution to normal distribution. See Figure-5.
\end{observation}

Figure-5 shows the variation of the losses of the samples during the training of the deep learning model. 
In the 1st graph, while the training has not started, the losses are distributed according to the normal distribution. 
In this case, after SGD is applied to the losses, the 2nd graph is formed. 
In the 2nd graph, the losses turn into a half normal distribution. 
If SGD is continued to be applied at this stage, the losses will remain in the half-normal distribution. 
However, at this stage, if ESG, which makes less error than SGD, is applied to the losses distributed with the half-normal distribution, the loss values of the samples will again resemble the normal distribution. 
In the CCL algorithm, this cyclical process continues throughout the training.

\begin{theorem}
The expected value of the error of  applying SGD and ESG  cyclically (CCL algorithm) is smaller than SGD.
\begin{equation}
    MSE(CCL) < MSE(SGD)
\end{equation}
\end{theorem}

\begin{proof}

From observation-1, the losses come from the normal distribution before the training starts.

From observation-2, if the losses come from the normal distribution, when SGD is applied, the losses are distributed according to the half-normal distribution, and when the losses are distributed according to the half normal distribution, when ESG is applied, the losses turn into the normal distribution.

According to Theorem-1, if the losses come from a normal distribution, the expected value of the error of applying SGD is lower than ESG.

According to Theorem-2, if the losses come from a half normal distribution, the expected value of the error of applying ESG is lower than SGD.

For $S(t) = 1$ CCL uses SGD when losses are normally distributed, and for $S(t) < 1$ CCL uses ESG when losses are half normally distributed. Because of this, CCL makes fewer errors than SGD.
\begin{equation}
    MSE(CCL) < MSE(SGD)
\end{equation}

\end{proof}

\begin{figure*}[]\label{region}
\includegraphics[width=\textwidth]{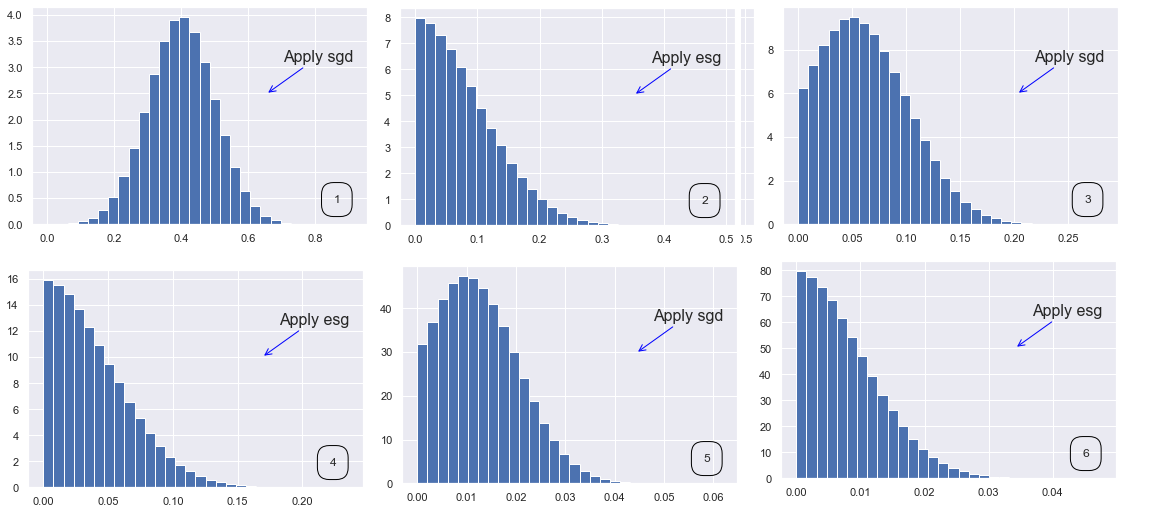}
\caption{Distribution of loss values according to different selection possibilities}
\end{figure*}

\begin{theorem}

If $r_i = \frac{1}{f_i\sum{\frac{1}{f_i}}}$ is defined for ESG, the expected value of CCL error is lower than SGD.

\end{theorem}

\begin{proof}

When we solve Equation-26 $x = f_i$ for $\lambda$, we get $\lambda = \frac{ln(x)}{x}$. This function's maximum value is $\frac{1}{e}$ for $x = e$.

\begin{equation}
exp(-\lambda.x) = \frac{1}{x}
\end{equation}

From teorem-2 ESG makes less error than SGD where $\lambda\sigma<\pi$
When we replace lambda with its maximum value we get

\begin{flalign}
\frac{1}{e}\sigma<\pi \\ \nonumber
\sigma<\pi.e
\end{flalign}

For values of $\sigma<\pi.e$, using $\frac{1}{x}$ instead of $exp(-\lambda.x)$ the expected value of CCL error is lower than SGD.

\end{proof}

The variance of the losses of all data sets used in the experiments is lower than the $\pi.e$ constant. Formally 

\begin{flalign}
\sigma<\pi.e
\end{flalign}

The variances of the losses obtained by the loss functions used in deep learning architectures comply with Equation 28. For this reason, it is also appropriate to use $\frac{1}{x}$ instead of  $exp(-\lambda.x)$ in CCL.

\section{Conclusions}
In this study, we were inspired by the earlier curriculum methods that the size of the training data set starts from small and gradually grows or starts with the whole training data set and decreases gradually; we examined the cyclical change of the training data set size.

We proposed an algorithm for the speed of this cycle change, start and end percentages. We chose the samples to be selected for the subset as the probabilistic ratio based on predetermined scores using the model itself. Then we train artificial neural networks using these values.

We tested this method on four image classification and fourteen text classification data sets. In the test sets, in ten out of eighteen data sets, we obtained statistically better accuracy values than the vanilla method. With the same cycling hyperparameters, success has been achieved in different types of datasets and architectures. However, even higher successes can be achieved with hyperparameters specific to the dataset/architecture.

We also give a theoretical explanation of why the CCL algorithm performs better. In some studies reported that CL methods could not outperform the vanilla method, and some reported that more successful results were obtained with the CL method. As shown in the theoretical analysis, under some conditions the vanilla method makes fewer mistakes, while under some conditions the curriculum methods make fewer mistakes. This situation brings an explanation to the literature. CCL on the other hand, makes fewer mistakes than the existing curriculum and vanilla methods because it uses the SGD algorithm when SGD makes fewer mistakes, and the ESG algorithm when ESG makes fewer mistakes.
CCL is more successful than both SGD and other CL methods because it uses the appropriate algorithm for the situation. 

The CCL algorithm predetermines the required dataset sizes in training and uses these dataset sizes throughout the training. However, during the training, dynamic size can be determined according to the current loss distribution instead of the determined sizes. This can lead to better optimization. Some studies suggest that even if CL does not increase final accuracy, it accelerates training at the beginning of training. Some studies, on the other hand, carry out updates by prioritizing more difficult examples towards the end of the training. With these strategies in mind, different strategies can be used for CL at the beginning, middle, and end of training.

\section*{Acknowledgment}
This study was supported by the Scientific and Technological Research Council of Turkey (TUBITAK) Grant No: 120E100.



\bibliography{citation}
\bibliographystyle{IEEEtran}

\end{document}